\documentclass[twocolumn]{article}
\usepackage{authblk}
\usepackage{lmodern}  
\usepackage[T1]{fontenc} 

\usepackage[backend=biber,maxbibnames=99,giveninits=true,style=alphabetic,url=false,sorting=none]{biblatex}

\usepackage{mathtools}
\usepackage{amsthm}
\usepackage{amssymb}
\theoremstyle{definition}
\newtheorem{definition}{Definition}
\newtheorem{theorem}{Theorem}
\newtheorem{lemma}{Lemma}
\newtheorem*{remark}{Remark}

\usepackage{multirow, booktabs}

\usepackage{enumerate}

\bibliography{string,references}%

\usepackage{placeins}
\usepackage{hyperref}
\hypersetup{
	breaklinks=true
}

\renewbibmacro{in:}{}

\DeclareFieldFormat{pages}{#1}
 \AtEveryBibitem{\ifentrytype{}{}{\clearfield{month}}}
 \AtEveryBibitem{\clearfield{url}}
 \AtEveryBibitem{\ifentrytype{}{}{\clearfield{isbn}}}
 \AtEveryBibitem{\ifentrytype{}{}{\clearfield{issn}}}
 \AtEveryBibitem{\clearfield{address}}
 \AtEveryBibitem{\ifentrytype{article}{}{\clearfield{publisher}}}

\renewbibmacro*{volume+number+eid}{%
   \ifentrytype{article}{%
 	\textbf{\printfield{volume}}%
    \iffieldundef{eid}{}{\addcomma\printfield{eid}}%
      }
 }

\DeclareFieldFormat{doilink}{%
\iffieldundef{doi}{#1}%
{\href{http://dx.doi.org/\thefield{doi}}{#1}}}

\DeclareBibliographyDriver{article}{%
  \usebibmacro{bibindex}%
  \usebibmacro{begentry}%
  \usebibmacro{author/translator+others}%
  \setunit{\labelnamepunct}\newblock
  \usebibmacro{title}%
  \newunit
  \printlist{language}%
  \newunit\newblock
  \usebibmacro{byauthor}%
  \newunit\newblock
  \usebibmacro{bytranslator+others}%
  \newunit\newblock
  \printfield{version}%
  \newunit\newblock
  \usebibmacro{in:}%
  \printtext[doilink]{%
  \printtext{%
    \printfield{journaltitle}%
    \setunit*{\addspace}%
    \mkbibbold{\printfield{volume}}%
    \setunit*{\addcomma\space}%
    \printfield{pages}%
    \setunit*{\space}%
    \printtext[parens]{}}%
  \newunit
  \usebibmacro{byeditor+others}%
  }%
  \newunit\newblock
  \iftoggle{bbx:isbn}
    {\printfield{issn}}
    {}%
  \newunit\newblock
  \usebibmacro{doi+eprint+url}%
  \newunit\newblock
  \usebibmacro{addendum+pubstate}%
  \setunit{\bibpagerefpunct}\newblock
  \usebibmacro{pageref}%
  \usebibmacro{finentry}}

\title{Bounds on the Generalization Error in Active Learning}
\author[1]{Vincent Menden}
\author[1]{Yahya Saleh\thanks{Corresponding author: yahya.saleh@uni-hamburg.de}}
\author[1]{Armin Iske}
\affil[1]{Department of Mathematics, Universität Hamburg, Bundesstr. 55, 20146, Hamburg, Germany}
\affil[ ]{\texttt{\{vincent.menden\}@studium.uni-hamburg.de}}
\affil[ ]{\texttt{\{yahya.saleh, armin.iske\}@uni-hamburg.de}}

\date{}
\begin{document}
\maketitle

\begin{abstract}
We establish empirical risk minimization principles for active learning by 
deriving a family of upper bounds on the generalization error. Aligning with empirical observations, the bounds suggest that superior query algorithms can be obtained by
combining both informativeness and representativeness query strategies, where the latter is assessed using integral probability metrics.
 To facilitate the use of these bounds in
application, we systematically link diverse active
learning scenarios, characterized by their loss functions and hypothesis
classes to their corresponding upper bounds. Our results show that
regularization techniques used to constraint the complexity of various hypothesis
classes are sufficient conditions to ensure the validity of the bounds.
The present work enables principled
construction and empirical quality-evaluation of query algorithms in active learning. 
\end{abstract}

\section{Introduction}

Empirical risk minimization principles are at the heart of statistical learning
theory. In addition to laying a formal mathematical foundation for supervised-learning algorithms, they lead to substantial advances in algorithmic design,
such as the development of max-margin methods~\cite{Vapnik:SLT1999,Mohri:FoundationsML:2018}. However, the majority of
empirical risk principles considered the standard passive supervised-learning
setting, and formal principles for other settings such as online or
semi-supervised learning are largely missing.

An important such setting is that of active learning (AL), where, similar to the
standard supervised-learning setting, computer oracles learn a probability
distribution that models a certain phenomenon given a finite set of observations. However, unlike in the standard
passive-learning setting,
the oracle in AL also selects an optimal, minimal set of observations to achieve this goal. Even in the age of big
data, numerous applications require this setting, mainly due to high
computational costs corresponding to the annotation, i.e., labeling of
datapoints~\cite{Settles:ActiveLearning:2009}. For example, in the emerging field of physics-informed neural
networks, it is often required to learn solutions or solution operators of
high-dimensional partial differential equations~\cite{Raissi:JCOP378:686,Kharazmi:arXiv1912:00873}. Generating the training data in
such learning tasks involve running computationally expensive numerical solvers. AL
is, indeed, a very appealing setting for such problems and has been extensively applied for, e.g., parameteric Schrödinger equations~\cite{Uteva:JCP149:174114, Saleh:JCP155:144109,Zhang:PRM3:023804}.

The crucial task in all AL scenarios is to \emph{query} the labels of the
most useful datapoints while minimizing the number of queries~\cite{Settles:ActiveLearning:2009}. The rationale
behind the design of such \emph{query algorithms} can be
divided into
two categories~\cite{Settles:ActiveLearning:2009, Dasgupta:TCS412:1767}. The
first category relies on the \emph{informativeness}
criterion~\cite{Freund:ANIPS1993:483,Seung:CLT5:287}, where the
query algorithm aims at selecting the most informative samples, whereby shrinking the
space of the candidate hypotheses as fast as possible. Such query algorithms indeed introduce a
sampling bias~\cite{Dasgupta:TCS412:1767}, as the selected training dataset is not necessarily
i.i.d. sampled from the true distribution. This renders the query algorithm
prone to oversampling outliers
that are not very representative of the application domain, where the model
would be employed~\cite{Kee:INSC454:401,
Settles:ActiveLearning:2009}. 
The second category is based on the \emph{representativeness} criterion, where the
query algorithm aims at selecting samples that are representative of the
patterns present in the unlabeled data. Such methods tend to perform well when
only a small labelled dataset is available, but their performance rather
deteriorate with increasing labeled-dataset size.
Numerous empirical and theoretical studies indeed point out that superior
query algorithms can be obtained by combining both
criteria~\cite{Dasgupta:TCS412:1767,Nguyen:ICML2004,Kee:INSC454:401}.

AL algorithms are often heuristic in designing the specific query criterion or
ad hoc in measuring the informativeness and representativeness of the samples.
Some first steps into a more principled
approach to AL were taken in~\cite{Wang:ACMTKDD9:23}, where the
authors derived an upper bound on the generalization error using the maximum mean
discrepancy (MMD) as a measure of the representativeness of a sample. Later, a
similar result was obtained using the Wasserstein distance as a measure of
representativeness~\cite{Shui:AISTATS2020:1308}.
However, these results assumed rather harsh conditions on the loss function and
the supervised-learning problem that restrict the applicability of these
upper bounds.

\textbf{Organization.} In \autoref{sec:preliminaries} we
cite the ERM principle in passive learning and introduce the notion of
integral probability metrics (IPMs). In
\autoref{ch3.1} we establish an ERM principle for AL. In \autoref{ch3.2} we link
the
upper bound in the ERM principle to two learning settings, employing linear models with the
$\ell_1$-loss function, and deep neural networks with the hinge loss,
respectively.

\section*{Notation}
On the probability measure space $(\Omega,\mathcal{A}, P)$ we consider the 
random vector $X: \Omega \to \mathbb{X} \subseteq
\mathbb{R}^n$ and the random variable $Y: \Omega \to \mathbb{Y} \subseteq
\mathbb{R}$. To simplify the terminology we refer to $X$ by a random variable
irrespective of the value of $n$. We set  
$Z = (X, Y)$ to be the joint random variable and denote by $P_Z$ its probability distribution on 
$\mathbb{Z}:=\mathbb{X} \times \mathbb{Y}$. We denote by $P_X$ the marginal
probability distribution and by $P_{Y|X}$ the conditional probability, i.e.,
$P_Z = P_X P_{Y|X}$. 

Given a random variable $Q: \Omega \to \mathbb{Q} \subseteq \mathbb{R}^n$, with
a distribution $P_Q$ and a measurable mapping $g: \mathbb{Q} \to \mathbb{R}$, we denote by
$\mathbb{E}_Q$ the expectation operator with respect to the distribution $P_Q$
and the random variable $Q$, i.e., $\mathbb{E}_{Q \sim P_Q}[g(Q)]$.

Throughout the paper we denote by 
$\mathfrak{H}$ a generic hypothesis class containing learners $h: \mathbb{X} \to
\mathbb{Y}$ and by $\ell:\mathbb{Y}^2 \to \mathbb{R}_{\geq 0}$ a generic loss
function that evaluates the deviation of a prediction $\hat{y} = h(x)$ from the
true label $y$. For such a loss function, a fixed $y\in \mathbb{Y}$ and a fixed
$h\in \mathfrak{H}$, we define $\ell^y:\mathbb{X}\to \mathbb{R}$ by $\ell^y(x):=\ell(y,h(x))$.

For a fixed $\mathfrak{H}$ and $\ell$ we denote by
$R_{\mathcal{P_Z}}(h)$ the true risk of a hypothesis $h\in \mathfrak{H}$ with
respect to $P_Z$, i.e.,
\begin{equation*}
{R_{\mathcal{P_Z}}(h)} := \int_{\mathbb{Z}} \ell(y, h(x)) \ d
P_Z(x,y). \end{equation*}
Given a dataset of finite observations $D_m:=\{z_1=(x_1, 
y_1),\dotsc,z_m=(x_m, y_m)\}$, we denote by $\hat{R}(h;D_m)$ the
empirical risk of the hypothesis
 $h$, i.e.,
\begin{align*}
  {\hat{R}(h;D_m)} := \frac{1}{m} \sum_{i=1}^m \ell(y_i, h(x_i)).
\end{align*}

Additionally, we define 
\begin{align*}
\mathcal{K} &:= \ell \circ \mathfrak{H} \circ D_m \\
  &:= \{\ell(y_i, h(x_i)): h\in \mathfrak{H}, (x_i, y_i) \in D_m\}.
\end{align*}

Finally, for a vector $v\in \mathbb{R}^n$ we denote by $\|v\|_2$ the standard
2-norm, i.e., $\|v\|_2=\sum_{i=1}^n \sqrt{w_i^2}$. Similarly, we set $\|v\|_1 = \sum_{i=1}^n|v_i|$ and for a matrix $M\in \mathbb{R}^{n\times m}$ we consider the spectral-2-norm $\|M\|_2:=\sup_{\|v\|_2=1}\|Mv\|_2$.
For compact sets $A\subset\mathbb{R}^n$ we set $M_A:=\max_{a\in A}\|a\|_2$.
\section{Preliminaries}
\label{sec:preliminaries}
In standard supervised learning, the unachievable goal of minimizing the true
risk is replaced by minimizing the empirical risk over a finite sample, while
imposing constraints on the complexity of the hypothesis class, often using
regularization techniques. 

Formally, this common practice in supervised learning can be understood as an
inductive principle, where the minimization of the true risk is replaced by the
minimization of an upper bound to it. Such upper bounds exist in a variety of
forms, often involving different notions of complexity of the
hypothesis class~\cite{Vapnik:SLT1999,Mohri:FoundationsML:2018, Shalev:ML:2014}. As an example, we cite the following celebrated result. 
\begin{theorem}
  \label{thm:ERM_PL}
  Assume that $\ell(y,h(x)) \leq k 
$ for some $k>0$, any $h \in \mathfrak{H}$ and any $(x,y) \in \mathbb{Z}$. Then, for any $\delta>0$, any $h \in \mathfrak{H}$, and some $c>0$, with probability of at least $1-\delta$ over
  the choice of the training set $D_m$ it holds that
  \begin{align}
    \label{eq:ERM_PL}
    R_{P_Z}(h)\leq &\hat{R}_{D_m\sim P_Z}(h)+
    2 \ \text{Rad}( \mathcal{K}) \nonumber  \\ 
    &+c \ \sqrt{\frac{2\log(\frac{4}{\delta})}{m}},
  \end{align}
  where $\text{Rad}(\mathcal{K})$ is the
  Rademacher complexity defined by 
  \begin{equation*} 
    \text{Rad}(\mathcal{K}) := \mathbb{E}_{\sigma}\left[\sup_{k\in \mathcal{K}}\frac{1}{m}\sum_{i=1}^m \sigma_i k(x_i)\right].
  \end{equation*}     
\end{theorem}
\begin{proof}
  See~\cite[Theorem. 26.5]{ Shalev:ML:2014}.
\end{proof}
Minimizing the upper bound in \eqref{eq:ERM_PL} was shown to be equivalent to
common supervised-learning practices across a variety of loss functions and
hypothesis classes. Moreover, such upper bounds were shown to accommodate novel
statistical behaviors, such as the generalization error of deep neural networks~\cite{Neyshabur:arXiv1805}.

Similar to the standard supervised-learning setting, the goal in AL
is to find a hypothesis of $h \in \mathfrak{H}$ that minimizes the true risk.
However, to achieve this goal, the oracle in AL is required to select a minimal
set of observations. This often violates the
passive-learning assumption that the training data is i.i.d. sampled from the true
distribution. Generally, the training data $D$ in AL follows the distribution
$
P_{\hat{Z}}:=P_Q \ P_{Y|X},
$
i.e., it shares the same conditional distribution as the true distribution
$P_Z$, but has a different marginal distribution $P_Q$. The choice of an optimal
query algorithm can, thus, be framed as finding an optimal marginal distribution $P_Q$.

The representativeness criterion in AL can be understood as the requirement that
$P_Q$ does not deviate too much from the true marginal $P_X$. To quantify this
deviation, we use the notion of IPM~\cite{Muller:AAP29:429}. 

\begin{definition}[Integral Probability Metrics]
Consider the measure space $(\mathbb{X}, \mathcal{B}(\mathbb{X}))$ where
$\mathcal{B}(\mathbb{X})$ denotes the Borel $\sigma$-algebra generated by
$\mathbb{X}\subset \mathbb{R}^n$. Further, let $\mathcal{F}\subseteq
\mathcal{B}_C$ with $\mathcal{B}_C$ the set of real-valued measurable functions
on $\mathbb{X}$, which are bounded by $C>0$.
Then, for two probability measures $P_X$ and $P_Q$ on $(\mathbb{X},
\mathcal{B}(\mathbb{X}))$ we define the integral probability metric as 
\begin{align}
  \label{eq:IPM}
d_{\mathcal{F}}(P_X, P_Q):=\sup_{f\in \mathcal{F}} \ \Bigl|&\int_{\mathbb{X}}f(x)dP_X(x) \nonumber \\&-\int_{\mathbb{X}}f(q)dP_Q(q)\Bigr|
\end{align}
\end{definition}
Choosing different generator classes $\mathcal{F}$ in \eqref{eq:IPM} leads to
different statistical distances. We consider the following two generators:
\begin{itemize}
  \item[(1)] The \textit{Total Variation metric} ($d_{\mathcal{F}_{\text{TV}}}$) is obtained by considering \begin{equation*}\mathcal{F}_{\text{TV}}:=\{f:\mathbb{X}\rightarrow
  \mathbb{R}:\left\lVert f\right\rVert_{\infty}\leq 1\},\end{equation*} 
  where $\left\lVert f\right\rVert_{\infty}$ denotes the supremum norm.
  \item[(2)]The \textit{Kantorovic metric} ($d_{\mathcal{F}_{K}}$) is obtained by
  considering \begin{equation*}\mathcal{F}_{K}:=\{f:\mathbb{X}\rightarrow
  \mathbb{R}:\left\lVert f\right\rVert_{L}\leq 1\},\end{equation*}
  where 
  $$\left\lVert f\right\rVert_{L}:= \sup  \ \left\{\frac{|f(x)-f(y)|}{\|x-y\|_2}: x\neq
  y, x,y \in S\right\}$$ 
  denotes the Lipschitz semi-norm on a metric space $(S,
  \rho)$. 
\end{itemize}

To establish the ERM principle for AL, we need the following concept.
\begin{definition}[Maximal Generator]
Let $\mathcal{F}\subseteq \mathcal{B}_C$ be a generator. We define the set of
maximal generators $\mathcal{R}_{\mathcal{F}}$ to be the set of functions $f\in\mathcal{B}_C$ with the
property
\begin{equation*}\label{eqmax}
\Bigl|\int_{\mathbb{X}}f(x)dP_X(x)-\int_{\mathbb{X}}f(q)dP_Q(q)\Bigr| \leq d_{\mathcal{F}}(P_X,P_Q),
\end{equation*}
for all probability measures $P_X$ and $P_Q$ on $(\mathbb{X},
\mathcal{B}(\mathbb{X}))$. 
\end{definition}
In other words, $\mathcal{R}_{\mathcal{F}}$ describes the largest set in
$\mathcal{B}_C$ preserving the value of $d_{\mathcal{F}}(\cdot,\cdot)$. It is
clear that $\mathcal{F}\subset \mathcal{R}_{F}$.
\begin{lemma}\label{lemma1}
Let $(\mathbb{Y},\mathcal{B}(\mathbb{Y}), P)$ be a probability space, $\mathcal{F}\subset\mathcal{B}_C$ a generator and $f:\mathbb{Y}\times\mathbb{X}\rightarrow \mathbb{R}$ a 
$\mathcal{B}(\mathbb{Y}\times\mathbb{X})$-measurable function with $f(y,\cdot)\in \mathcal{F}\subset \mathcal{B}_C$ for all $y\in\mathbb{Y}$.
Then
\begin{equation*}
g(\cdot):=\int_{\mathbb{Y}}f(y,\cdot)dP(y)
\end{equation*}
is a well-defined function on $\mathbb{X}$ and it holds that $g\in\mathcal{R}_{\mathcal{F}}$.
\end{lemma}
\begin{proof}
See~\cite[Theorem 3.4]{Muller:AAP29:429}.
\end{proof}
Note that \autoref{lemma1} also holds for any $f\in \mathcal{R}_{\mathcal{F}}$. The stage is now ready to state our results.
\section{ERM in Active Learning}
We begin by establishing the ERM principle for AL, where the IPM is
used as a measure of representativeness.

\subsection{Bounding the True Risk}\label{ch3.1}
We recall that the training data in AL is assumed to follow a distribution
$P_{\hat{Z}}$ that shares the same conditional distribution of
the generating distribution $P_Z$, i.e., $P_{\hat{Z}}=P_Q P_{Y|X}$. Further,
recall that a given a loss
function $\ell:\mathbb{Y}^2 \to \mathbb{R}_{\geq 0}$ induces the function
$\ell^{y}:\mathbb{X}\to \mathbb{R}$ by $\ell^{y}(x):=\ell(y,h(x))$ for some $y \in
\mathbb{Y}$ and $h\in \mathfrak{H}$.
  
\begin{theorem}[ERM principle for AL]
\label{thm:ERM_AL}
Let $\mathcal{F}\subset\mathcal{B}_C$ be a generator for some $C>0$, and
$\ell:\mathbb{Y}^2\rightarrow\mathbb{R}_{\geq 0}$ be a loss function that
satisfies the hypothesis of \autoref{thm:ERM_PL}. Further, let
$\ell^y \in\mathcal{F}$ for all $y\in\mathbb{Y}$ and $h\in\mathfrak{H}$ and $\hat{D}_m =\{\hat{Z}_1,\dotsc\, \hat{Z}_m\}\sim
P_{\hat{Z}}$
be an i.i.d sample. Then, with probability of at least $1-\delta$ and for any
$h\in \mathfrak{H}$, we have
\begin{equation}
\begin{split}
R_{P_Z}(h) &\leq \hat{R}_{\hat{D}_m\sim P_{\hat{Z}}}(h) + d_{\mathcal{F}}(P_X,P_Q)\\
&+ 2 \ \text{Rad} (l\circ \mathfrak{H} \circ \hat{D}_m)\\ 
&+ c\sqrt{\frac{2\log(\frac{4}{\delta})}{m}}.
\end{split}
\end{equation}
\end{theorem}
\begin{proof}
  We note that the hypothesis of this theorem satisfies the conditions of
  \autoref{thm:ERM_PL}. Therefore, it follows that
  \begin{equation}
  \begin{split}
  R_{P_Z}(h)\leq & R_{P_Z}(h)-R_{P_{\hat{Z}}}(h)\\
  &+ \hat{R}_{\hat{D}_m\sim P_{\hat{Z}}}(h)\\
  &+2 \ \text{Rad}(l\circ \mathfrak{H} \circ \hat{D}_m)\\
  &+c\sqrt{\frac{2\log(\frac{4}{\delta})}{m}}.
  \end{split}
  \end{equation}
  Set $K(h):=R_{P_Z}(h)-R_{P_{\hat{Z}}}(h)$ and note that
  \begin{equation*}
  \begin{split}
  K(h) = &\int_{\mathbb{X}}\int_{\mathbb{Y}} l(y, h(x)) dP_{Y|X}(y)dP_X(x) \\
  &- \int_{\mathbb{X}}\int_{\mathbb{Y}} l(y, h(x)) dP_{Y|X}(y)dP_Q(x) \\
  =& \int_{\mathbb{X}}\int_{\mathbb{Y}} l(y, h(x)) dP_{Y|X}(y)dP_X(x) \\
  &- \int_{\mathbb{X}}\int_{\mathbb{Y}} l(y, h(x)) dP_{Y|X}(y)dP_Q(x)
  \end{split}
  \end{equation*}
  by virtue of Fubini's theorem.
  Set 
  \begin{equation}
    \label{eq:g}
    g:=\int_{\mathbb{Y}} l(y, h(\cdot)) dP_{Y|X}(y)  
  \end{equation}
  
   and note that $\ell(y,h(\cdot))=\ell^y$ satisfies all the hypotheses of
   \autoref{lemma1} and hence $g\in \mathcal{R}_\mathcal{F}$.
   Thus, using the definition of $\mathcal{R}_{\mathcal{F}}$ we can estimate
  \begin{equation*}
  \begin{split}
  D(h) &= \int_{\mathbb{X}}g(x)dP_X(x) - \int_{\mathbb{X}}g(x)dP_Q(x) \\
  &\leq \sup_{f\in\mathcal{R}_{\mathcal{F}}} \ \Bigl|\int_{\mathbb{X}}f(x)dP_X(x) - \int_{\mathbb{X}}f(x)dP_Q(x)\Bigr| \\
  &=\sup_{f\in\mathcal{F}} \ \Bigl|\int_{\mathbb{X}}f(x)dP_X(x) - \int_{\mathbb{X}}f(x)dP_Q(x)\Bigr| \\
  &= d_{\mathcal{F}}(P_X,P_Q). 
  \end{split}
  \end{equation*}
  \normalsize
  \end{proof}
  \begin{remark} We note that an upper bound on the true risk in AL using the
  IPM appeared in the work of~\cite{Wang:ACMTKDD9:23}. However, to derive their
  result the authors
 made a direct assumption on $g$, see \eqref{eq:g}. A more refined
  version appeared in the work of~\cite{Saleh:thesis:2023}, where the author
  derived direct conditions on the loss function $\ell$ that would reduce the
  IPM to the Kantorovic metric and the MMD. \autoref{thm:ERM_AL} can be considered as a
  more general formulation of these results that allow a direct connection to the
  literature on maximal generators. 
  \end{remark}  

\autoref{thm:ERM_AL} establishes an empirical risk principle, which is in accordance with common practices in AL. To see this consider
a classification task and assume that
the AL oracle has access to a hypothesis class $\mathfrak{H}$, an
initially labelled dataset $D^{(0)}\sim P_Z$ and a pool of unlabeled data that
is i.i.d. sampled from $P_X$. The upper bound suggests finding a hypothesis $h$ and
sampling an additional dataset $D^{(1)}$ that minimize the empirical risk. A
certain hypothesis $h$ that minimizes the empirical risk on $D^{(0)}$ would
benefit the most from a dataset $D^{(1)}$ that is close to the decision
boundary. This corresponds to the concept of informativeness sampling in AL.
In addition, the upper bound in the theorem suggests that a query strategy should
sample points, whose distribution is close to the true marginal distribution of
the data. In other words, an optimal query strategy should sample points that
are representative of the underlying marginal. Indeed, a balance between these
two criteria is crucial for the success of an AL query algorithm~\cite{Dasgupta:TCS412:1767,Nguyen:ICML2004,Kee:INSC454:401}. 
\subsection{Mapping Learning Settings to Generalization Bounds}\label{ch3.2}
The upper bound derived in \autoref{thm:ERM_AL} is generic and can take many
forms by choosing different generator classes $\mathcal{F}$. We aim in this
section at deriving explicit bounds for the true risk given a certain learning
setting. We consider the learning setting to be
determined by a choice of the hypothesis class $\mathfrak{H}$ and the loss function
$\ell$. In the following
$\mathbb{X}\subset\mathbb{R}^n$ and $\mathbb{Y}\subset\mathbb{R}$ unless otherwise specified.

We consider first a regression task employing the linear hypothesis class
\begin{equation*}
  \mathfrak{H}_L:=\{h:\mathbb{X}\rightarrow\mathbb{Y}:h(x)=w^Tx+b, w\in\mathbb{R}^n,b\in \mathbb{R} \},
  \end{equation*}
  where $w$ and $b$ are the learnable parameters along with the loss
  $\ell_1(y,h(x)):=|y-h(x)|$ defined for any $y\in \mathbb{Y}$ and $h\in
  \mathfrak{H}_L$.
  
\begin{theorem}[Linear Hypothesis Classes]\label{thm:linearhc}
  Consider a regression problem employing $\mathfrak{H}_L$ and the
 $\ell_1$-loss. Assume that $w$ is such that $\|w\|_2\leq 1$. Then
      the true risk of a hypothesis $h\in \mathfrak{H}_L$ can be bounded as in
      \autoref{thm:ERM_AL} by choosing the generator class $\mathcal{F} =
      \mathcal{F}_{K}$.
  \end{theorem}
  \begin{proof}
    Analogous to
  our previous notation we set $\ell_1^y(x) :=
  \ell_1(y, h(x))$ for any $y\in \mathbb{Y}$.

  Fix $y \in \mathbb{Y}$ and $h\in\mathfrak{H}_L$. By \autoref{thm:ERM_AL}, it suffices to show that $\ell_1^y\in \mathcal{F}_K$.
      For any $x_1,x_2\in\mathbb{X}$, it holds that
    \begin{equation*}
    \begin{split}
    |\ell_1^y(x_1)&-\ell^y_1(x_2)| = \bigl||h(x_1) -y|-|h(x_2)-y|\bigr| \\
      &\leq \left|(w^Tx_1 + b )- (w^Tx_2 + b)\right| \\
      &\leq \left|w^T (x_1 - x_2)\right| \\
      &\leq \|w\|_2 \ \|x_1 - x_2\|_2,
    \end{split}
    \end{equation*}
     where we used the reversed-triangle inequality and the Cauchy-Schwarz inequality.
     Setting $\|w\|_2\leq1$ implies that $\|\ell_1^y\|_L \leq 1$ and hence $\ell_1\in \mathcal{F}_K$.
\end{proof}
\autoref{thm:ERM_PL} suggests that the natural regularization constraint
$\|w\|_2\leq 1$, commonly used for mitigating overfitting, is sufficient to bound the true risk of a linear hypothesis
class in an AL setting.
\begin{table*}[t]
  \centering
  \renewcommand{\arraystretch}{1.5}
  \begin{tabular}{c| c| c c c}
  \hline
  & $\mathfrak{H}$ & $\ell$ & Condition & IPM \\
  \hline
   & $\mathfrak{H}_L$& $\ell_1$ & $\|w\|_2\leq 1$ & $d_{\mathcal{F}_K}$ \\ 
  \cline{3-5}
  \text{Regression}& & $\ell_2$ & $\|w\|_2\leq \frac{1-M_{\mathbb{Y}}-|b|}{M_{\mathbb{X}}}$ & $d_{\mathcal{F}_\text{TV}}$ \\
  \cline{2-5}
  & $\mathfrak{H}_g$ & $\ell_1$ & $\frac{2M_{\mathbb{X}}}{\sigma^2}\|w\|_1^2\leq 1$ & $d_{\mathcal{F}_K}$ \\ 
  \hline
   & $\mathfrak{H}_{\sigma(L)}$ & $\ell_\text{log}$ & $\|w\|_2\leq\frac{\log(e-1)}{M_{\mathbb{X}}}$ & $d_{\mathcal{F}_{\text{TV}}}$ \\ 
  \cline{2-5}
  \text{Classification}& $\mathfrak{H}_{\text{SVM}}$ & $\ell_H$ & $\|w\|_2\leq 1$ & $d_{\mathcal{F}_K}$ \\
  \cline{2-5}
  & $\mathfrak{H}_{\text{NN}}$ & $\ell_H$ & $\|o\|_2\prod_{i=1}^{L}\|W\|_2\leq 1$ & $d_{\mathcal{F}_K}$ \\ 
  \hline
  \end{tabular}
  \caption{The table summarizes the mapping of various learning settings to
   corresponding IPMs in \autoref{thm:ERM_AL} under specified conditions on the learnable
  parameters $w$. The learning tasks are characterized by the hypothesis class
  (linear $\mathfrak{H}_L$, Gaussian $\mathfrak{H}_g$, logistic
  $\mathfrak{H}_{\sigma(L)}$, support vector machines $\mathfrak{H}_{\text{SVM}}$, and neural networks $\mathfrak{H}_{\text{NN}}$)
  and the loss function $\ell$ ($\ell_1$, logistic $\ell_\text{log}$,
   and hinge $\ell_H$). The formal definitions of the hypothesis classes and the
   losses are provided in \autoref{ch3.2} and \autoref{app:full_results}.
  }
  \label{tab:full_width_table}
  \end{table*}
\newpage
  We now look at an example of a binary classification problem, i.e., $\mathbb{Y} = \{-1,1\}$, using feed-forward neural
  networks
  \begin{equation*}
    \mathfrak{H}_{\text{NN}}:=\{h:\mathbb{X}\rightarrow\mathbb{Y}:h(x)=\text{sign}(o^Tf(x)+b)\}
  \end{equation*}
  with weight $o\in\mathbb{R}^n$ and bias $t\in\mathbb{R}$ in the output layer and the neural network function
  $f(x) = W^{(L)}\sigma (W^{(L-1)}\cdots\sigma(W^{(1)}x+b^{(1)})\cdots
  +b^{(L-1)})+b^{(L)}$, where $\sigma$ is the ReLU activation function, and
  $W^{(l)}$, $b^{(l)}$ are the weight matrices and bias vectors, respectively. The
  learnable parameters are assumed to have arbitrary
  finite dimensions. We consider the hinge loss $\ell_H(y,h(x))=\max \ (0,1-y(w^Tx+b))$.
\begin{theorem}[Neural Networks] \label{thm:NN}
  Consider a binary classification task employing $\mathfrak{H}_{\text{NN}}$ and the
  $\ell_H$-loss. Assume that
  $\|o\|_2\prod_{i=1}^{L}\|W\|_2\leq 1$, then the true risk of a hypothesis $h\in \mathfrak{H}_{\text{NN}}$ can be bounded as in
 \autoref{thm:ERM_AL} by choosing the generator class $\mathcal{F} =
 \mathcal{F}_{\text{K}}$.
\end{theorem}
\begin{proof}
    Similarly to the previous proof, it suffices to show that
    $\ell^y_{\text{H}}\in \mathcal{F}_{\text{K}}$ for any $y\in
    \mathbb{Y}=\{-1,1\}$ and $h\in\mathfrak{H}_{\text{NN}}$ with
    $\|o\|\prod_{i=1}^{L}\|W\|_2\leq 1$. This follows directly from the fact
    that feedforward neural networks with ReLU activation functions are Lipschitz continuous with
    bounded Lipschitz constant  $\|o\|\prod_{i=1}^{L}\|W\|_2$, see~\cite[Proposition 1]{Scaman:arXiv1805.10965},
    and the fact that $\ell_h$ is Lipschitz continuous with Lipschitz constant 1. 
\end{proof}  
We note that \autoref{thm:NN} as well suggests that regularization
constraints on the learnable parameters are sufficient to bound the true risk in
an AL setting.  

\autoref{thm:linearhc} and \autoref{thm:NN} are only two examples on how to
constraint the hypothesis class for deriving a suitable generalization bound in
an AL setting. We note that a variety of other learning settings employing other
losses and other hypotheses classes can be considered. We summarize
similar results that allow embedding in various generator classes in
\autoref{tab:full_width_table} and refer the reader to the respective proofs in \autoref{app:full_results}.
\FloatBarrier
   
\section{Conclusion and Outlook}
We derived a bound on the generalization error for AL that is based on the IPM as a measure
of representativeness. The bound suggests that a query strategy should sample
informative samples while maintaining a distribution of the queried samples that
is close to the true marginal distribution. This aligns with common practices in
AL. 

This result can be used as a theoretical insight to guide the design of query
strategies in AL. To facilitate such use, we augmented the bound with a variety
of examples that show how to embed different learning settings in various
generator classes. A key insight from these examples is that the regularization
constraints on the learnable parameters seem to play a crucial role for a
principled design of query strategies.

Additionally, our results can be used to evaluate the quality
of ad hoc query strategies in AL. A necessary step towards such an application is
to derive upper bounds to the true risk that employ empirical estimates of the
IPM, see, e.g.,~\cite{Bharath:EJS6:1550} for general discussion on
empirical estimates of IPMs.

We note that the choice of the IPM as a measure of representativeness is not unique. Other choices of metrics to measure the representativeness of the
samples, such as, e.g., $\phi-$divergences, can be considered~\cite{Bharath:arXiv0901:2698}. We leave this as an open question for future research.
\printbibliography
\newpage
\appendix
\section{More Learning Settings}
\label{app:full_results}

\autoref{thm:linearhc} and \autoref{thm:NN} are only two examples on how to
constraint the hypothesis class for deriving a suitable generalization bound in
an AL setting. In the following we provide the reader with further results,
which are also summarized in \autoref{tab:full_width_table}.
  
We start by looking at the Gaussian hypothesis class defined as
  \begin{equation*}\mathfrak{H}_{g}:=\Bigl\{h:\mathbb{X}\rightarrow\mathbb{Y}: h(x) = \sum_{i=1}^{n}w_i g(x,t_i)\Bigr\},
  \end{equation*}
  where $g(x,t_i) = e^{\left(-\frac{\|x - t_i\|^2}{2\sigma^2}\right)}$, for some fixed $\sigma>0$, and learnable parameters $w=(w_1,\dots ,w_n)\in\mathbb{R}^n$ and $t_i\in\mathbb{X}$ for any $i=1,\dots,n$.
  \begin{theorem}[Gaussian Hypothesis Classes]\label{thm:ghc}
    Consider a regression problem employing $\mathfrak{H}_{\text{g}}$ and the
    $\ell_1$-loss. Assume $\mathbb{X}$ to be compact, with bound
    $M_{\mathbb{X}}$, and $w$ to be such that
    $\frac{2M_{\mathbb{X}}}{\sigma^2}\|w\|_1 \leq 1$,
        then the true risk of a hypothesis $h\in \mathfrak{H}_g$ can be bounded as in
        \autoref{thm:ERM_AL} by choosing the generator class $\mathcal{F} =
        \mathcal{F}_{K}$.
    \end{theorem}
    \begin{proof}
      Set $\ell_{1}^y(x) :=
    \ell_{1}(y, h(x))$ for any $y\in \mathbb{Y}$. Fix $y \in \mathbb{Y}$ and $h\in\mathfrak{H}_{g}$.
    By \autoref{thm:ERM_AL}, it suffices to show that $\ell_{1}^y\in \mathcal{F}_{\text{K}}$.
      For arbitrary $t_i\in\mathbb{X}$ observe that
    \begin{equation*}
      \begin{split}
        \left|\frac{\partial}{\partial x} g(x, t_i)\right| &= \frac{1}{\sigma^2} e^{\left(-\frac{\|x - t_i\|^2}{2\sigma^2}\right)} \|x - t_i\|_2\\
        &\leq \frac{2M_{\mathbb{X}}}{\sigma^2}.
      \end{split}
      \end{equation*}
    Thus, the function $f(x)=g(x,t_i)$ is Lipschitz continuous on $\mathbb{X}$ with $\|f\|_L\leq\frac{2M_{\mathbb{X}}}{\sigma^2}$.
    Finally, for any $x_1,x_2\in \mathbb{X}$ we have
    \begin{equation*}
    \begin{split}
    |\ell_{1}^y(x_1)&-\ell_{1}^y(x_2)| = \bigl||h(x_1) -y|-|h(x_2)-y|\bigr| \\
      &\leq \left|\sum_{i=1}^{n}w_i g(x_1,t_i) - \sum_{i=1}^{n}w_i g(x_2,t_i)\right| \\
      &\leq \sum_{i=1}^{n} |w_i| \left|g(x_1,t_i) - g(x_2,t_i)\right| \\
      &\leq \sum_{i=1}^{n} |w_i| \frac{2M_{\mathbb{X}}}{\sigma^2} \|x_1 - x_2\|_2 \\
      &= \|w\|_1\frac{2M_{\mathbb{X}}}{\sigma^2}\|x_1 - x_2\|_2.
    \end{split}
    \end{equation*}
    Setting $\|w\|_1\frac{2M_{\mathbb{X}}}{\sigma^2}\leq1$ implies that $\|\ell_{1}^y\|_L \leq 1$ and hence $\ell_{1}\in \mathcal{F}_K$.
    \end{proof}
    Another result for the linear hypothesis classes is the following.
    \begin{theorem}[Linear Hypothesis Classes]\label{thm:linearhc2}
      Consider a regression problem employing $\mathfrak{H}_{L}$ and the
      $\ell_2$-loss. Assume $\mathbb{X}$ and $\mathbb{Y}$ to be compact, with bounds $M_{\mathbb{X}}$ 
      and $M_{\mathbb{Y}}$, and  $w$ and $b$ to be such that $\|w\|_2\leq \frac{1-M_{\mathbb{Y}}-|b|}{M_{\mathbb{X}}}$,
            then the true risk of a hypothesis $h\in \mathfrak{H}_{L}$ can be bounded as in
            \autoref{thm:ERM_AL} by choosing the generator class $\mathcal{F} =
            \mathcal{F}_{\text{TV}}$.
      \end{theorem}
      \begin{proof}
        Set $\ell_{\text{H}}^y(x) :=
        \ell_{2}(y, h(x))$ for any $y\in \mathbb{Y}$. Fix $y \in \mathbb{Y}$ and $h\in\mathfrak{H}_{L}$.
        By \autoref{thm:ERM_AL}, it suffices to show that $\ell_{2}^y\in \mathcal{F}_{\text{TV}}$.
        To this end, it suffices to show that $|y-w^Tx-b|\leq 1$.
        Note that
        \begin{equation*}
          \begin{split}
          |y-w^Tx-b| &\leq |y|+\|w\|_2\|x\|_2 + |b| \\
          &\leq M_{\mathbb{Y}}+\|w\|_2M_{\mathbb{X}} + |b|,
          \end{split}
          \end{equation*}
          Thus, setting $\|w\|_2\leq \frac{1-M_{\mathbb{Y}}-|b|}{M_{\mathbb{X}}}$ we get
         $\ell_{2}^y\in \mathcal{F}_{\text{TV}}$.
        \end{proof}
        Next, we look at some more binary classification settings.

  Consider the hypothesis class of logistic linear functions
\begin{equation*}
  \mathfrak{H}_{\sigma(L)}:=\{h:\mathbb{X}\rightarrow\mathbb{Y}:h(x)=\sigma(w^Tx), w\in\mathbb{R}^n\}
\end{equation*}
with the sigmoid activation function $\sigma(z) =\frac{1}{1+e^{-z}}$ for $z\in\mathbb{R}$ and learnable parameter $w$. This hypothesis class is often used in combination with the 
logistic loss function $\ell_{\log}(y,h(x))= -(y \log(h(x))+(1-y)\log(1-h(x)))$
for $y\in \mathbb{Y},x\in \mathbb{X}$ and $h\in \mathfrak{H}_{\sigma(L)}$. We
set $\mathbb{Y} = \{0,1\}$ and denote by $e$ the Euler number.
\begin{theorem}[Logistic Hypothesis Classes]\label{thm:logistichc}
  Consider a binary classification problem employing $\mathfrak{H}_{\sigma(\text{L})}$ and the
  logistic loss. Assume $\mathbb{X}$ to be compact, with bound $M_{\mathbb{X}}$,
  and $w$ to be such that 
  $\|w\|_2\leq \log(e-1)M_{\mathbb{X}}$, then the true
    risk of a hypothesis $h\in \mathfrak{H}_{\sigma(L)}$ can be bounded as in
        \autoref{thm:ERM_AL} by choosing the generator class $\mathcal{F} =
        \mathcal{F}_{\text{TV}}$.
  \end{theorem}
  \begin{proof}
    Set $\ell_{\text{log}}^y(x) :=
    \ell_{\text{log}}(y, h(x))$ for any $y\in \mathbb{Y}$. Fix $y \in \mathbb{Y}$ and $h\in\mathfrak{H}_{\sigma(L)}$.
    By \autoref{thm:ERM_AL}, it suffices to show that $\ell_{\text{log}}^y\in \mathcal{F}_{\text{TV}}$.
    We first consider $y = 1$ and observe that for any $x\in\mathbb{X}$, we have $
    |\ell^y_{\text{log}}(x)| = \log \ (1+e^{-w^Tx})$. Similarly, for $y=0$, we
    observe that $
    |\ell^y_{\text{log}}(x)| = \log \ (1+e^{w^Tx})$ for any $x\in\mathbb{X}$.
      Thus, setting $\|w\|_2\leq \frac{\log(e-1)}{M_{\mathbb{X}}}$ implies $\|\ell^y_{\text{log}}\|_\infty\leq 1$ and hence $\ell^y_{\text{log}}\in \mathcal{F}_{\text{TV}}$.
  \end{proof}
  Next we look at the hypothesis class of linear support vector machines
(SVM) given by 
\begin{equation*}
  \mathfrak{H}_{\text{SVM}}:=\{h:\mathbb{X}\rightarrow\mathbb{Y}:h(x)=\text{sign}(w^Tx+b)\}
\end{equation*}
with learnable parameters $w\in\mathbb{R}^n$ and $b\in\mathbb{R}$. The primary loss function used in linear SVM is the hinge loss
$\ell_H(y,h(x)):=\max(0,1-y(w^Tx+b))$ for $y\in\mathbb{Y}, x\in\mathbb{X}$ and
$h\in\mathfrak{H}_{\text{SVM}}$. We set $\mathbb{Y}=\{-1,1\}$.
\begin{theorem}[Support Vector Machines]\label{thm:SVM}
  Consider a binary classification problem employing $\mathfrak{H}_{\text{SVM}}$ and the
  hinge loss. Assume $w$ to be such that $\|w\|\leq 1$,
        then the true risk of a hypothesis $h\in \mathfrak{H}_{\text{K}}$ can be bounded as in
        \autoref{thm:ERM_AL} by choosing the generator class $\mathcal{F} =
        \mathcal{F}_{\text{K}}$.
  \end{theorem}
  \begin{proof}
    Set $\ell_{\text{H}}^y(x) :=
    \ell_{\text{H}}(y, h(x))$ for any $y\in \mathbb{Y}$. Fix $y \in \mathbb{Y}$ and $h\in\mathfrak{H}_{\text{SVM}}$.
    By \autoref{thm:ERM_AL}, it suffices to show that $\ell_{\text{H}}^y\in \mathcal{F}_{\text{K}}$.
    We note that the function $\ell_{H}^y(\hat{y}):=\max(0,1-y\hat{y})$ is Lipschitz
      continuous on $\mathbb{Y}$ with Lipschitz constant 1. Additionally, for any $w\in\mathbb{R}^n$ and $b\in\mathbb{R}$ the affine linear function $w^Tx+b$ is Lipschitz continuous on $\mathbb{X}$ with Lipschitz constant $\|w\|_2$. Thus,
      for any $x_1,x_2\in \mathbb{X}$ we have 
      \begin{equation*}
        \begin{aligned}
          |\ell_{\text{H}}^y(x_1)-\ell_{\text{H}}^y(x_2)| &\leq \bigl|(w^Tx_1+b)-(w^Tx_2+b)\bigr| \\
          &\leq \|w\|_2 \|x_1 - x_2\|_2.
        \end{aligned}
      \end{equation*}
      Setting $\|w\|_2\leq1$ implies that $\|\ell_{\text{H}}^y\|_L \leq 1$ and hence $\ell_{\text{H}}\in \mathcal{F}_K$.
  \end{proof}    
\end{document}